\documentclass[a4paper,12pt]{article}

\usepackage[top=2.5cm,bottom=2.5cm,left=2.5cm,right=2.5cm]{geometry}

\usepackage{indentfirst}
\usepackage[dvips]{graphics}
\usepackage{subfigure}
\usepackage{graphicx}
\usepackage{epsfig}
\usepackage{hyperref}

\usepackage{framed}
\usepackage{psfrag}
\usepackage{tikz} 

\usepackage{amssymb,amscd,amsmath,amsthm}
\everymath{\displaystyle}
\newcommand{\R}{\mathbb{R}}

\renewcommand{\H}{\mathcal{H}}

\usepackage{accanthis}

\newtheorem{Definition}{Definition}
\newtheorem{Theorem}{Theorem}

\begin{document}

\title{A Convergence indicator for Multi-Objective Optimization Algorithms}

\author{ Thiago Santos%
	     \thanks{Email: santostf@ufop.edu.br} and 
	     Sebastião Xavier%
	     \thanks{Email: sermax@ufop.edu.br}\\ 
	     Federal University of Ouro Preto, Ouro Preto, MG, Brazil 
}

\maketitle

\begin{abstract}
	
	{\bf Abstract}. The algorithms of multi-objective optimization had a relative growth in the last years. 
	Thereby, it's requires some way of comparing the results of these. In this sense, performance 
	measures play a key role. In general, it's considered some properties of these algorithms such as 
	capacity, convergence, diversity or convergence-diversity. There are some known measures such as 
	generational distance (GD), inverted generational distance (IGD), hypervolume (HV), 
	Spread($\Delta$), Averaged Hausdorff distance ($\Delta_p$), R2-indicator, among others. In this 
	paper, we focuses on proposing a new indicator to measure convergence based on the traditional 
	formula for Shannon entropy. The main features about this measure are: 1) It does not require tho know 
	the true Pareto set and 2) Medium computational cost when compared with Hypervolume.
	
	
\end{abstract}

\section{Introduction}
Nowadays, the evolutionary algorithms (EAs) are used to obtain approximate solutions of multi-objective optimization problems (MOP) and these EAs are called multi-objective evolutionary algorithms (MOEAs). Some of these algorithms are very well-known among the community such that NSGA-II (See  \cite{DPAM:02}), 
SPEA-II (See  \cite{Zitzler01spea2:improving}), MO-PSO (See  
\cite{CoelloCoello:2002:MPM:1251972.1252327}) and MO-CMA-ES (See  
\cite{Igel:2007:CMA:1245371.1245373}).
Although the most of MOEAs to use the previous criteria, when $m>3$ the MOP 
is called {\em Many Objectives Optimization Problems} and in this case algorithms Pareto-Based 
is not good enough. Some papers try to explain the reason why such thing happens (See  
\cite{TFSantos:2016,Schutze:2011}). 
To avoid the phenomenon caused by Pareto relation, some researchers indicates others way of 
comparative the elements (See  \cite{Li:2015}) or change into Non-Pareto-based MOEAs, such as 
indicator-based and aggregation-based approaches (See  \cite{Ishibuchi:2010,Wagner:2007}).

With the MOEAs in mind, it's natural to know how theirs outputs are relevant. According the \cite{ZhouQLZSZ:11}, it was listed $23$ indicators which intend to provide us some information about this such as algorithms. Those indicators have basically three goals, that is: 1) Closeness to the theoretical Pareto set, 
2) Diversity of the solutions or 3) number of Pareto-optimal solutions.  
Build an indicator which give us information about all those three goals above it is something tough to do it. 
However, there are great ways to measure the quality in this path. The most used one, it is the Inverted Generational Distance (IGD) and  Generational Distance (GD)(see \cite{IGD:98}) because of simplicity and low cost to calculate. Recently, another one based on this measure was  propose so called Hausdorff Measure (see \cite{Schutze:2012}) which combines the IGD and GD and takes its maximum. These indicator is efficient to obtain informations about closeness of the output of some algorithm with the True Pareto set. The difficult here is because in the order to calculate the IGD/GD it's necessary to know the True Pareto Set of the problem and some ( or the most of time) such as information it's not available.
Another one, very well-know, is the Hypervolume or S-metric (see \cite{HubandetAl2006}). The main problem with this indicator is related with its huge computational cost and to avoid this some authors suggest to use Monte Carlo Simulation to approximate the value and decrease  the cost (see \cite{Bader:2011}).

Here, we will be providing an indicator which allow us talk about nearness of the True Pareto Set. The idea comes from the \cite{SSW:02} what introduce a great function that satisfies the KKT conditions. With this function, we do not need to know anything about the exact solutions of the problem or needs to choose a right reference point.

The paper is organised as follow: on section \ref{sMOP} we establish the general multi-objective problem and on \ref{sSome} we talk about three well-known indicators. Finally, on section \ref{sH} we presents our idea and on \ref{sSimu} we do some numeral simulations.

\section{Multi-objective Problem (MOP)}\label{sMOP}
It's common to define multi-objective optimisation 
problems (MOP) as follows:
\begin{equation}
\begin{array}{l}
\min f(x), ~~ 
f(x)=(f_1(x),f_2(x),\cdots,f_m(x))\\
\mbox{s.t:} ~~ 
x \in \Omega \subset \mathbb{R}^n
\end{array}
\label{MOP}
\end{equation}
in which $x \in \mathbb{R}^n$ is the {\em decision variable vector}, 
$f(x) \in \mathbb{R}^m$ is the {\em objective vector}, and $\Omega \subseteq 
\mathbb{R}^n$ is the {\em feasible set} that we consider as compact and connected region. In this 
work, we assumed here that the functions $f_i(\cdot)$ 
are continuously differentiable (or $\mathcal{C}^2(\Omega)$).

The aim of multi-objective optimisation is to obtain an estimate of the set of points 
belonging to $\Omega$ which minimize, in the certain sense that we will call by 
Pareto-optimality. 
\begin{Definition}
	Let $u, v\in \mathbb{R}^m$. We say that $u$ dominates $v$, denoted by $ u \preceq v$, iff \,  
	$\forall i=1,\ldots,m$
	\begin{equation}
	u_i \leq v_i, \, \textmd{and \,} u \neq v.
	\end{equation}
\end{Definition}

\begin{Definition}
	A feasible solution $x^* \in \Omega$ is a Pareto-optimal solution of 
	problem (\ref{MOP}) if there is no $x \in \Omega$ such that $f(x) \preceq 
	f(x^*)$. The set of all Pareto-optimal solutions of problem \eqref{MOP} is called by
	Pareto Set (PS) and the its image is the Pareto Front (PF). Thus,
	
	\begin{eqnarray*}
		PS &=& \{x \in \Omega|\, \nexists y \in \Omega, f(y) \preceq f(x)\} \\
		PF &=& \{f(x)| x \in PS\}
	\end{eqnarray*}
\end{Definition}

A classical work (See \cite{kuhn1951}) established a relationship between the 
points of PS and gradient informations from the problem \eqref{MOP}. That connection it is
known by Karush-Kuhn-Tucker (KKT) conditions for Pareto optimality that we define as follow:
\begin{Theorem}[KKT Condition \cite{kuhn1951}]
	\label{kkt}
	Let $x^* \in PS$ of the problem\eqref{MOP}. Then, there exists nonnegative scalars $\lambda_i\geq 
	0$, with $\sum_{i=1}^m \lambda_i =  1$, such that
	\begin{equation}
	\sum_{i=1}^m \lambda_i \nabla  f_i(x^*) = 0
	\end{equation}
\end{Theorem}
This theorem will be fundamental to this paper because we will use this fact to formulate our 
proposal.

\section{Some Convergence Indicators}\label{sSome}
In this section, we will go to relate some metrics or indicators known by scientific 
community and well-done established. The idea here is to compare those indicators further below with our proposal measure.
\subsection{GD/IGD}
The Inverted Generational Distance (IGD) indicator has been using since 1998 when it was created. 
The IGD measure is calculated on objective space, which can be viewed as an approximate distance 
from the 
Pareto front to the solution set in the objective space. In the order to define this metric, we 
assume that the set 
$\Lambda= \{y_1, y_2, \cdots, y_r\}$
is an approximation of the Pareto front for the problem \eqref{MOP}. Let be the $V_{MOEA}$ a 
solutions set obtained by some MOEA in the objective space as 
$V_{MOEA}=\{v_1, v_2, \cdots, v_k\}$ where $v_i$ is a point in the objective space. Then, the 
IGD metric is calculated for the set $V_{MOEA}$ using the reference points $\Lambda$ as follows:
\begin{equation}
IGD(V_{MOEA},\Lambda)=\frac{1}{r}\left (  \sum_{i=1}^r d(y_i,V_{MOEA})^2 \right)^{1/2},
\label{IGDmetric}
\end{equation}
where $d(y,X)$ denotes the minimum Euclidean distance between $y$ and the points in $X$. Besides, we also can to define de $GD$ metric by
\begin{equation}
GD(V_{MOEA},\Lambda)=IGD(\Lambda, V_{MOEA})
\label{GDmetric}
\end{equation}

This indicators are measure representing how "far" the approximation front is from the true Pareto front. Lower values of GD/IGD represents a better performance. The only different between GD and IGD is that in the last one you don't miss any part true Pareto set on comparison.

In \cite{Ishibuchi:14} indicates two main advantages about it: 1) its computational 
efficiency even many-objective problems and 2) its generality which usually shows the overall 
quality of an obtained solution set. The authors in \cite{Ishibuchi:14} studied some difficulties 
in specifying reference points to calculate the IGD metric.

\subsection{Averaged Hausdorff Distance $( \Delta_p )$}
This metric combine generalized versions of $GD$ and $IGD$ that we denote by $GD_p/IGD_p$ 
and defined by, with the same previous notation, 
\begin{eqnarray}
IGD(V_{MOEA},\Lambda)_p & = & \left ( \frac{1}{r} \sum_{i=1}^r 
d(y_i,V_{MOEA})^p \right )^{1/p}\\
GD(V_{MOEA},\Lambda)_p & = & IGD(\Lambda, V_{MOEA})p
\label{igd_p}
\end{eqnarray}

The indicator, so called by $\Delta_p$, was proposed in \cite{Schutze:2012} defined by
\begin{equation}
\Delta_p(X,Y)=max\{IGD(X,Y)_p,GD(X,Y)_p \},
\label{delta_p}
\end{equation}
In \cite{Schutze:2012} the author proved that function is a semi-metric, $\Delta_p$ does not 
fill the triangle inequality, for $1\leq p < \infty$. Many others properties was proved in 
\cite{Schutze:2012}.

\subsection{Hypervolume (HV)}
This indicator has been using by the community since 2003. Basically, the hypervolume of a set of solutions measures the size of the portion of objective space that is dominated by those solutions as a group. In general, hypervolume is favored because it captures in a single scalar both the closeness of the solutions to the optimal set and, to some extent, the spread of the solutions across objective space. There are many works on this indicator such as in \cite{HubandetAl2006} which the author studied how expensive to calculate this indicator was. Few years later, it was proposed a faster alternative by using Monte Carlo simulation ( See  \cite{Bader:2011}) that it was addressed for many objectives problem by Monte Carlo simulation. In the order to get a right definition, you can look at \cite{Bader:2011,HubandetAl2006}.

\section{Proposal Measure $\H$}\label{sH}
In the order to define our proposal measure, consider the quadratic optimization problem 
\eqref{POQ} associated with \eqref{MOP}:
\begin{equation}
\min_{\alpha \in \R^n} \left \{   \left \|\sum_{i=1}^m \alpha_i \nabla  f_i(x) \right \|^2 
; \,\,\, \alpha_i\geq 0, \,\,\,  \sum_{i=1}^m \alpha_i=1 \right \}
\label{POQ}
\end{equation}
The existence and uniqueness of a global solution of the problem \eqref{POQ} is stabilised  
in \cite{SSW:02}. The function $q:\R^n \rightarrow \R^n$ given by
\begin{equation}
q(x)=\sum_{i=1}^m \widehat{\alpha}_i \nabla  f_i(x)
\label{eq:direcao}
\end{equation}
where $\widehat{\alpha}$ is a solution of the (\ref{POQ}), becomes well defined.
There is a interested property  about this function that was proved in \cite{SSW:02}:  
\begin{itemize}
	\item 	each $x^*$ with $\|q(x^*)\|^2=0$, where $|\cdot|$ represents euclidean norm,  
	fulfills the first-order necessary conditions for Pareto optimality given by 
	Theorem \ref{kkt}.
\end{itemize}
Thereby, these points are certainly Pareto candidates what motivates the next definition about 
nearness. 
\begin{Definition}
	A point $x \in \Omega$ is called $\epsilon-$closed to Pareto set if $\|q(x^*)\|^2 <  
	\epsilon$.
\end{Definition}

Let the set $X=\{x_1,x_2,\cdots,x_k\}$ the output from some evolutionary algorithms. With 
the feature about that function, we can to define a new measure by:
\begin{equation}
\mathcal{H}(X):=\frac{1}{2k}\sum_{i=1}^{k}-q_i\log_2(q_i)
\label{Hmeasure}
\end{equation}
where $q_i=min\{1/exp(1), \|q(x_i)\|^2\}$ and put $0\cdot log_2(0)=0$.

The expression \eqref{Hmeasure} is the traditional formula used for Shannon Entropy( See  
\cite{Gray:1990:EIT:90455}). Unlike the original way that is used this as a entropy, in the our metric 
each $q_i$ is not to related with a probability space.

\begin{Theorem}
	About the function in \eqref{Hmeasure}, we have that 
	\begin{equation}
	0\leq \mathcal{H}(X) \leq \frac{1}{2} \frac{\log_2(e^1)}{e^1}
	\end{equation}
\end{Theorem}
\begin{proof}
	First, since $0 \leq q_i \leq 0.5$ then $\mathcal{H}(X)\geq 0$. On the other hand, 
	it is known that the function $f(x)=-x\log_2(x)$ has a maximum at $x=1/exp(1)$, hence
	\begin{eqnarray*}
		\mathcal{H}(X) &=& \frac{1}{2k}\sum_{i=1}^{k}-q_i\log_2(q_i) \\
		&\leq& \frac{1}{2k}\sum_{i=1}^{k}-\exp(-1)\log_2(\exp(-1))\\
		&\leq& \frac{1}{2k}(k)(\exp(-1)\log_2(\exp(1))\\
		&=&\frac{1}{2} \frac{\log_2(\exp(1))}{\exp(1)} \approx0.26537
	\end{eqnarray*}
\end{proof}

A reasonable the interpretation of about $\mathcal{H(\cdot)}$ might be: $(a)$ whether $\mathcal{H}(X)$ closed to or equal $0$ then the set $X$ has a good convergence to the Pareto Set; $(b)$ whether $\mathcal{H}(X)$ closed to or equal $0.26537$ then the set $X$ does 
not convergence to Pareto Set.

The main features associated with this metric are: 1) The true Pareto set doesn't require to be 
predefined and 2) It can be to used with many-objectives problems. The first point is important 
because it is almost impossible to know the true PF in general. On the other hand, the second 
attribute is useful since it doesn't exists many indicator to be used in this problems.
\begin{figure}[!htb]
	\centering
	\includegraphics[scale=0.7]{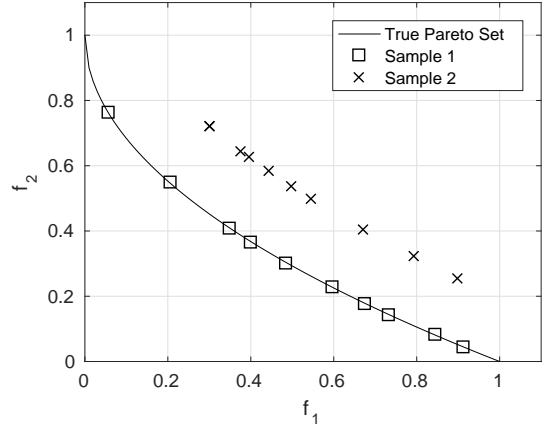}
	\caption{Empirical Idea to illustrate $\H$.}
	\label{example1}
\end{figure}

In the figure \eqref{example1}, we have an empirical with two samples that supposed approximates the 
true Pareto set (PF). Intuitively, we can say that the sample 1 is more closed to PF than sample 2. 
If it were calculated IGD/GD metric, the sample 1 we will get values more near the zero than sample 
2. Also we have the same conclusion with our proposal.

\subsection{Computational Complexity}
To calculate $\mathcal{H}$, the major computational cost is related with calculation of function $q$ because we 
have that previously to find the solution of \eqref{POQ}. This requires $\mathcal{O}(m\cdot n)$ computations to compute it. Consequently, the overall complexity needed is $\mathcal{O}(m\cdot n \cdot k)$.

The aim behind of figures \eqref{example2} and \eqref{example3} it's to show a comparative with others metrics by measuring CPU time. We conducted the experiment on Intel Core I7 with 16gb RAM and we used the traditional benchmark function
DTLZ2 (see \cite{dtlz2002a}). 

With the simulation on figure \eqref{example2}, we fixed three objective functions and set the size of approximate population (popSize) between $100 \leq popSize \leq 2000$ and for the calculation of $\Delta_p$ we need a exact population and for this we set the size of exact population as $4*popSize$. The same we did for the calculation the HV in this case. On the other hand, for the simulation on figure \eqref{example3}, we configured popSize as $50*M$ and size of exact population as $100*M$ where $M$ is the number of objectives function on each run.

\begin{figure}[!htp]
	\centering
	\includegraphics[scale=0.55]{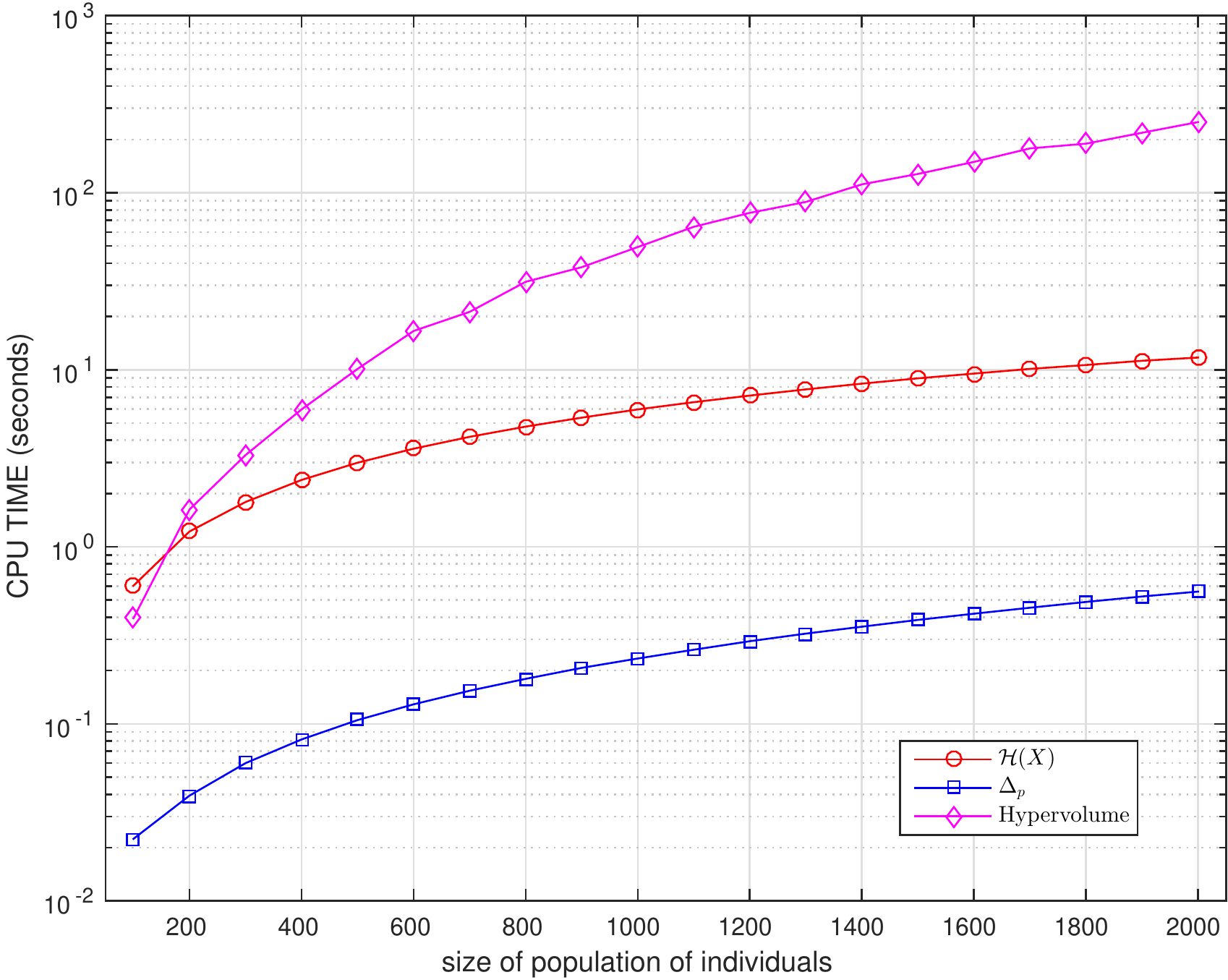}
	\caption{CPU TIME( in logarithmic scale) versus size of population. }
	\label{example2}
\end{figure}

\begin{figure}[!htb]
	\centering
	\includegraphics[scale=0.55]{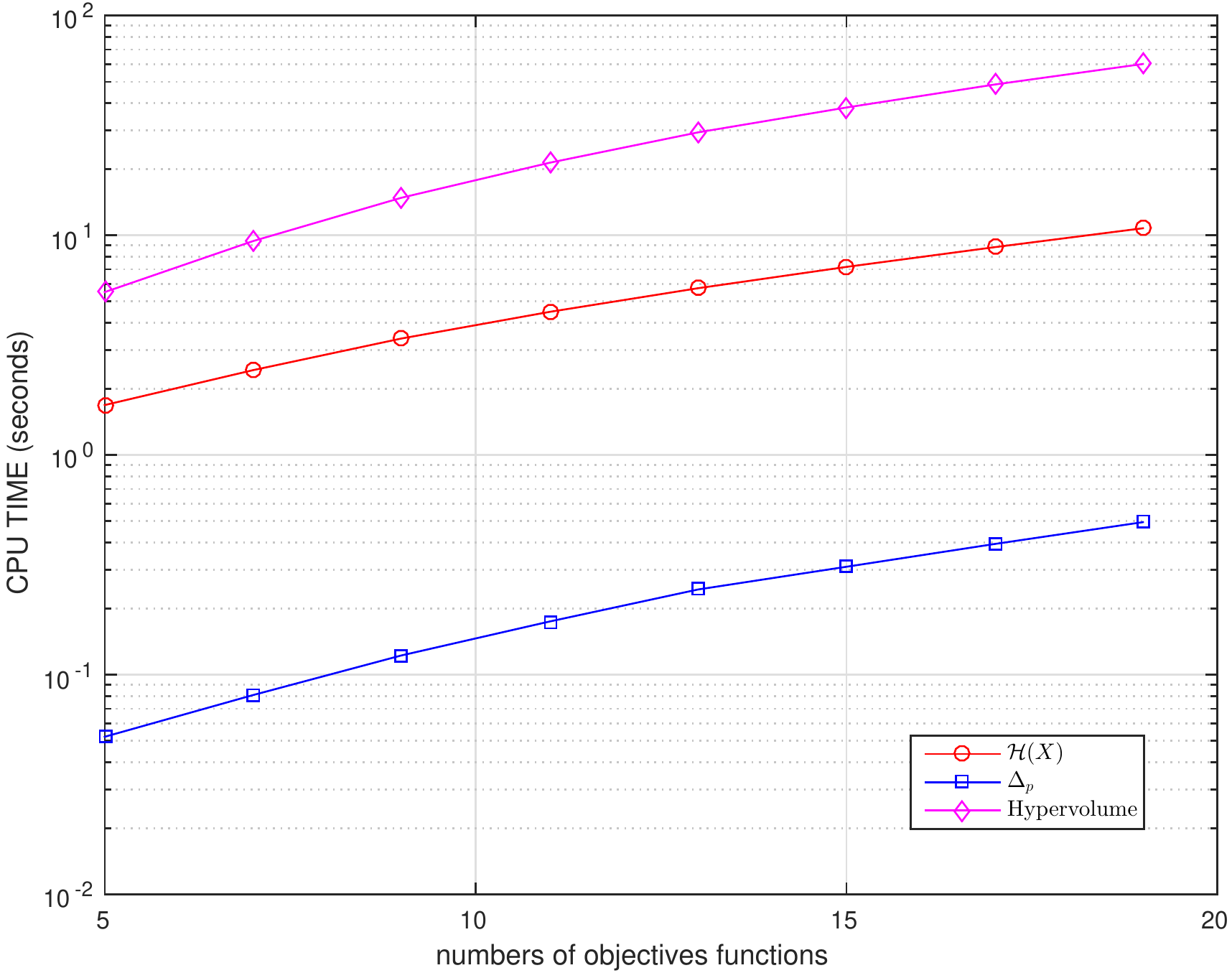}
	\caption{CPU TIME( in logarithmic scale) versus numbers of objective functions.}
	\label{example3}
\end{figure}

\begin{figure}[!htb]
	\centering
	\includegraphics[scale=0.55]{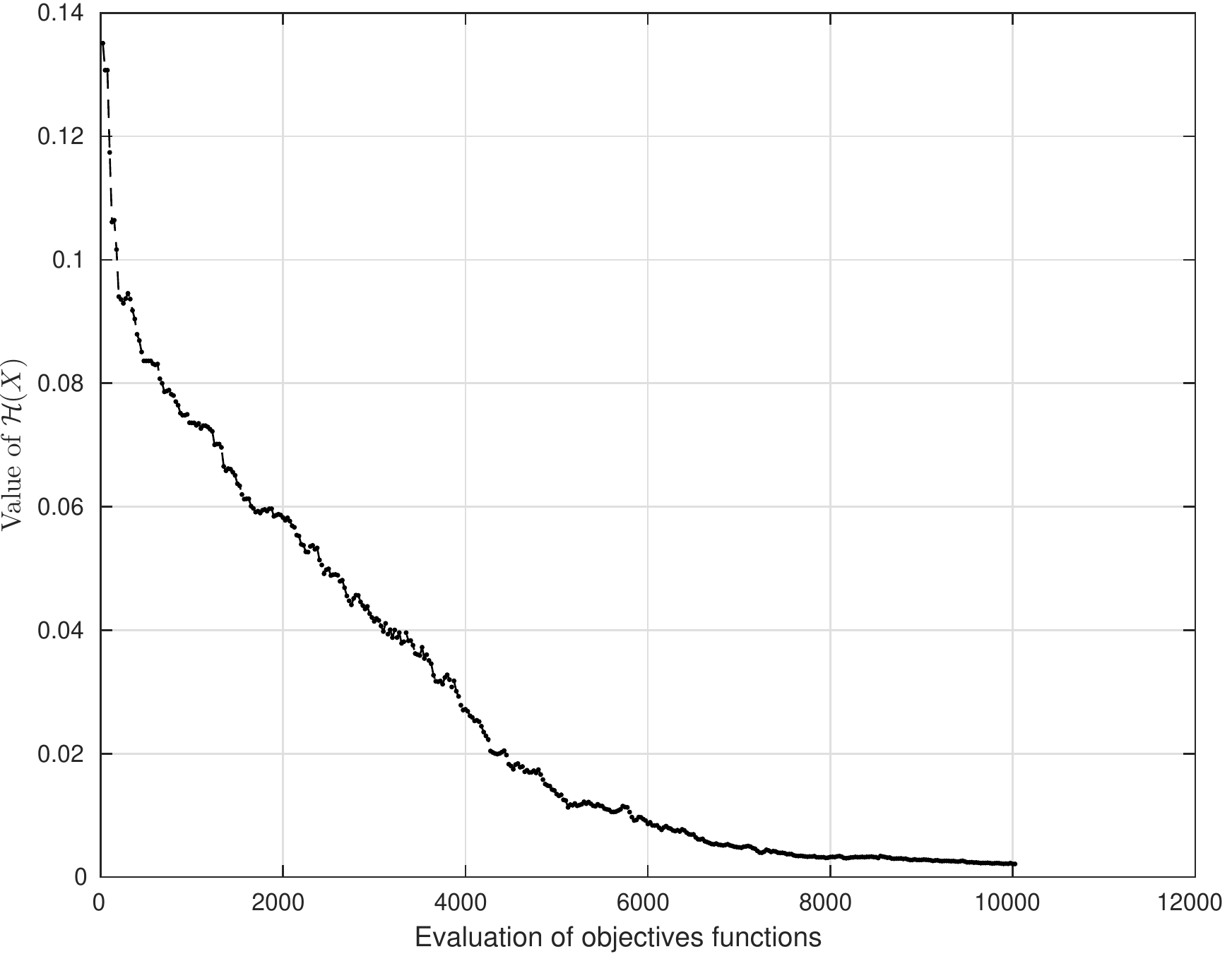}
	\caption{Entropy measure is computed for non-dominated solutions of MOPSO-cd populations 
		for Problem ZDT2.}
	\label{example4}
\end{figure}

Both tests showed us a most cost for calculating of the our indicator in relation the $\Delta_p$ but 
relatively less when we compare with the calculation of HV. A fact that already was expected because of the equation \eqref{POQ} as we pointed out earlier. Nevertheless, from the figure \eqref{example3}, it is possible to conclude by CPU time that for many objectives problems our proposal is acceptable in relation with the HV.

The figure \eqref{example4} shows the behavior of $\mathcal{H}$ with the non-dominated solutions of 
MOPSO-cd algorithm (see \cite{mopsocd}) for the benchmark function ZDT2. Through this experiment, which was fixed $10^5$ evaluation function, it is possible to see that the values of this measure decrease monotonically. Actually, this property comes from the continuity and monotonicity of the function $f(x)=-x\log_2(x)$ into the interval $[0,1/e^1]$.

\section{Simulation Results}\label{sSimu}
The aim here is to compare our proposal with other indicators presented previously, i.e, $\Delta_2$ e HV. Besides, we have been chosen benchmark  DTLZ ( DTLZ1, DTLZ2, DTLZ5 and  DTLZ7) due the facility of its scalarization. We have decided for the three well-known algorithms which the setup follow below:
\begin{enumerate}
	\item MOPSO-CD	(See \cite{mopsocd} ) 
	\item NSGA-II  		(See \cite{DPAM:02}) 
	\item NSGA-III  	    (See		\cite{NSGA3}) 
\end{enumerate}

In both algorithms, we fixed the output population in $100$ individuals with three objectives functions and we ran $30$ times with $25000$ function evaluations running in each one. We have leaded this experiment on PlatEMO framework (See \cite{platEmo}).

\begin{table}[!htp]
	\centering
	\small
	\caption{DTLZ1 function results}
	\begin{tabular}{|c|c|c|c|}
		\hline 
		\rule[-1ex]{0pt}{2.5ex}  			& $\Delta_2$  & HV 			& $\H$ \\ 
		\hline 
		\rule[-1ex]{0pt}{2.5ex} NSGA-II 	&  		$ 0.030712 $ ( $ 0.024266 $ ) &  $ 1.296030 $ ( $ 0.027097 $ ) & $ 0.246989 $ ( $ 0.006442 $ )   \\ 
		\hline 
		\rule[-1ex]{0pt}{2.5ex} NSGA-III 	&  	{\boldmath $ 0.021337 $ ( $ 0.001355 $ )} &  {\boldmath $ 1.303803 $ ( $ 0.000829 $ )} &  {\boldmath $0.226985 $ ( $ 0.002771 $ )}  \\ 
		\hline 
		\rule[-1ex]{0pt}{2.5ex} MOPSO-CD    &  		$ 8.760441 $ ( $ 2.893284 $ ) &  $ 0.007216 $ ( $ 0.039522 $ ) & $ 0.262165 $ ( $ 0.005288 $ )  )  \\ 
		\hline 
	\end{tabular} 
	\label{rDTLZ1}
\end{table}

\begin{table}[!htp]
	\centering
	\small
	\caption{DTLZ2 function results}
	\begin{tabular}{|c|c|c|c|}
		\hline 
		\rule[-1ex]{0pt}{2.5ex}  					& $\Delta_2$  & HV 			& $\H$ \\ 
		\hline 
		\rule[-1ex]{0pt}{2.5ex} NSGA-II 	   &  	$ 0.066306 $ ( $ 0.003177 $ ) &  $ 0.707364 $ ( $ 0.005279 $ ) & $ 0.038143 $ ( $ 0.004933 $ )   \\ 
		\hline 
		\rule[-1ex]{0pt}{2.5ex} NSGA-III 	  &  {\boldmath	$ 0.054834 $ ( $ 0.000627 $ )} & {\boldmath  $ 0.744441 $ ( $ 0.000132 $ )} & {\boldmath $ 0.002395 $ ( $ 0.000631 $ )}   \\ 
		\hline 
		\rule[-1ex]{0pt}{2.5ex} MOPSO-CD &  	$ 0.077568 $ ( $ 0.003564 $ ) &  $ 0.661220 $ ( $ 0.012205 $ ) & $ 0.129108 $ ( $ 0.011345 $ )  \\ 
		\hline 
	\end{tabular} 
	
	\label{rDTLZ2}
\end{table}

\begin{table}[!htp]
	\centering
	\small
	\caption{DTLZ5 function results}
	\begin{tabular}{|c|c|c|c|}
		\hline 
		\rule[-1ex]{0pt}{2.5ex}  					& $\Delta_2$  & HV 			& $\H$ \\ 
		\hline 
		\rule[-1ex]{0pt}{2.5ex} NSGA-II 	   &  	{\boldmath $ 0.005768 $ ( $ 0.000314 $ )} &   {\boldmath $0.437257 $ ( $ 0.000299 $ )} & {\boldmath $ 0.002650 $ ( $ 0.000656 $ )}   \\ 
		\hline 
		\rule[-1ex]{0pt}{2.5ex} NSGA-III 	  &  	$ 0.013543 $ ( $ 0.001737 $ ) &  $ 0.429002 $ ( $ 0.002105 $ ) & $ 0.002882 $ ( $ 0.000711 $ )   \\ 
		\hline 
		\rule[-1ex]{0pt}{2.5ex} MOPSO-CD &  	$ 0.006982 $ ( $ 0.000640 $ ) &  $ 0.434430 $ ( $ 0.000654 $ ) & $ 0.015889 $ ( $ 0.002332 $ )   \\ 
		\hline 
	\end{tabular} 
	
	\label{rDTLZ5}
\end{table}

Throughout the tables \eqref{rDTLZ1}, \eqref{rDTLZ2} and \eqref{rDTLZ5} we can conclude the same thing with all indicator, included our proposal as it was expected by the previous section.

\section{Conclusions}\label{sCon}
In this paper, we have introduced a new indicator $\mathcal{H}$	which goal is evaluate the outcome from an evolutionary algorithm on the Multi-objective optimisation problems. We have tested the new approach with some classic benchmark and make a comparative with some others indicators that is already known. By experimental results, we can to conclude the same what is indicated by the other indicators of performance. Unlike such these indicators, our proposal needs to known nothing about the true Pareto of set of the MOP. This features, we consider a great to thing because most MOPs is doesn't have this information. 

For the future, we will try to decrease the condition about the function of the problems, until now we required being $\mathcal{C}^2(\Omega)$.

\section*{Acknowledgment}
The authors would like to be thankful for the FAPEMIG by financial support to this project.

\bibliographystyle{plain}
\bibliography{mybib}

\end{document}